\documentclass[sigconf,nonacm]{aamas} 

\usepackage{balance} 
\usepackage{tabularray}
\usepackage{algorithm}
\usepackage{algpseudocode}

\setcopyright{ifaamas}
\acmConference[AAMAS '25]{Proc.\@ of the 24th International Conference
on Autonomous Agents and Multiagent Systems (AAMAS 2025)}{May 19 -- 23, 2025}
{Detroit, Michigan, USA}{A.~El~Fallah~Seghrouchni, Y.~Vorobeychik, S.~Das, A.~Nowe (eds.)}
\copyrightyear{2025}
\acmYear{2025}
\acmDOI{}
\acmPrice{}
\acmISBN{}

\acmSubmissionID{<<EasyChair submission id>>}

\title[AAMAS-2025 Formatting Instructions]{Adaptive Budget Optimization for Multichannel Advertising Using Combinatorial Bandits}

\author{Briti Gangopadhyay}
\email{briti.gangopadhyay@sony.com}
\affiliation{%
  \institution{Sony}
  \country{Japan}
  }

\author{Zhao Wang}
\email{zhao.wang@sony.com}
\affiliation{%
  \institution{Sony}
  \country{Japan}
}

\author{Alberto Silvio Chiappa}
\email{alberto.chiappa@epfl.ch}
\affiliation{%
  \institution{Sony*\thanks{*Work done during internship at Sony}, EPFL}
  \country{Japan, Switzerland}
}

\author{Shingo Takamatsu}
\email{shingo.takamatsu@sony.com}
\affiliation{%
  \institution{Sony}
  \country{Japan}
}

\begin{abstract}
Effective budget allocation is crucial for optimizing the performance of digital advertising campaigns. However, the development of practical budget allocation algorithms remain limited, primarily due to the lack of public datasets and comprehensive simulation environments capable of verifying the intricacies of real-world advertising. While multi-armed bandit (MAB) algorithms have been extensively studied, their efficacy diminishes in non-stationary environments where quick adaptation to changing market dynamics is essential. In this paper, we advance the field of budget allocation in digital advertising by introducing three key contributions. First, we develop a simulation environment designed to mimic multichannel advertising campaigns over extended time horizons, incorporating logged real-world data. Second, we propose an enhanced combinatorial bandit budget allocation strategy that leverages a saturating mean function and a targeted exploration mechanism with change-point detection. This approach dynamically adapts to changing market conditions, improving allocation efficiency by filtering target regions based on domain knowledge. Finally, we present both theoretical analysis and empirical results, demonstrating that our method consistently outperforms baseline strategies, achieving higher rewards and lower regret across multiple real-world campaigns.
\end{abstract}


\newcommand{\BibTeX}{\rm B\kern-.05em{\sc i\kern-.025em b}\kern-.08em\TeX}

\begin{document}


\pagestyle{fancy}
\fancyhead{}


\maketitle 


\section{Introduction}

Digital advertising is a fast growing area of research, with the global market size approaching \$700 billion in 2024 and projected to surpass \$830 billion by 2026 \cite{10.1145/3637528.3671476}. In 2023, the average internet user spent around 6.5 hours daily engaging with content largely driven by advertisement. In the United States, digital advertising expenditure reached \$189 billion in 2021 \cite{iab1}, showing a significant 35\% year-over-year growth, driven in part by the COVID-19 pandemic. Despite economic challenges such as high inflation and rising interest rates, digital advertising continued to expand, reaching \$225 billion in 2023 \cite{iab2}. 

As the advertising sector continues to evolve, the number of sub-campaigns within a portfolio grows \cite{giglimulti}, driven by diversification across various formats (e.g., Search, Display, Video) and platforms (e.g., Google, Meta). To ensure profitability from delivering a diverse portfolio of campaigns, it is crucial to manage digital marketing budgets effectively (Fig \ref{budgetallocation}). This resource allocation problem has attracted significant interest from the machine learning community \cite{rangaswamy2005opportunities} as logged data can be procured from different business campaigns. The presence of rich features in this data further fuels the development of automated decision-making systems, as learning algorithms are often better equipped to interpret multidimensional tabular data than human intuition.

\begin{figure}[!h]
\centering
\includegraphics[width=0.45\textwidth]{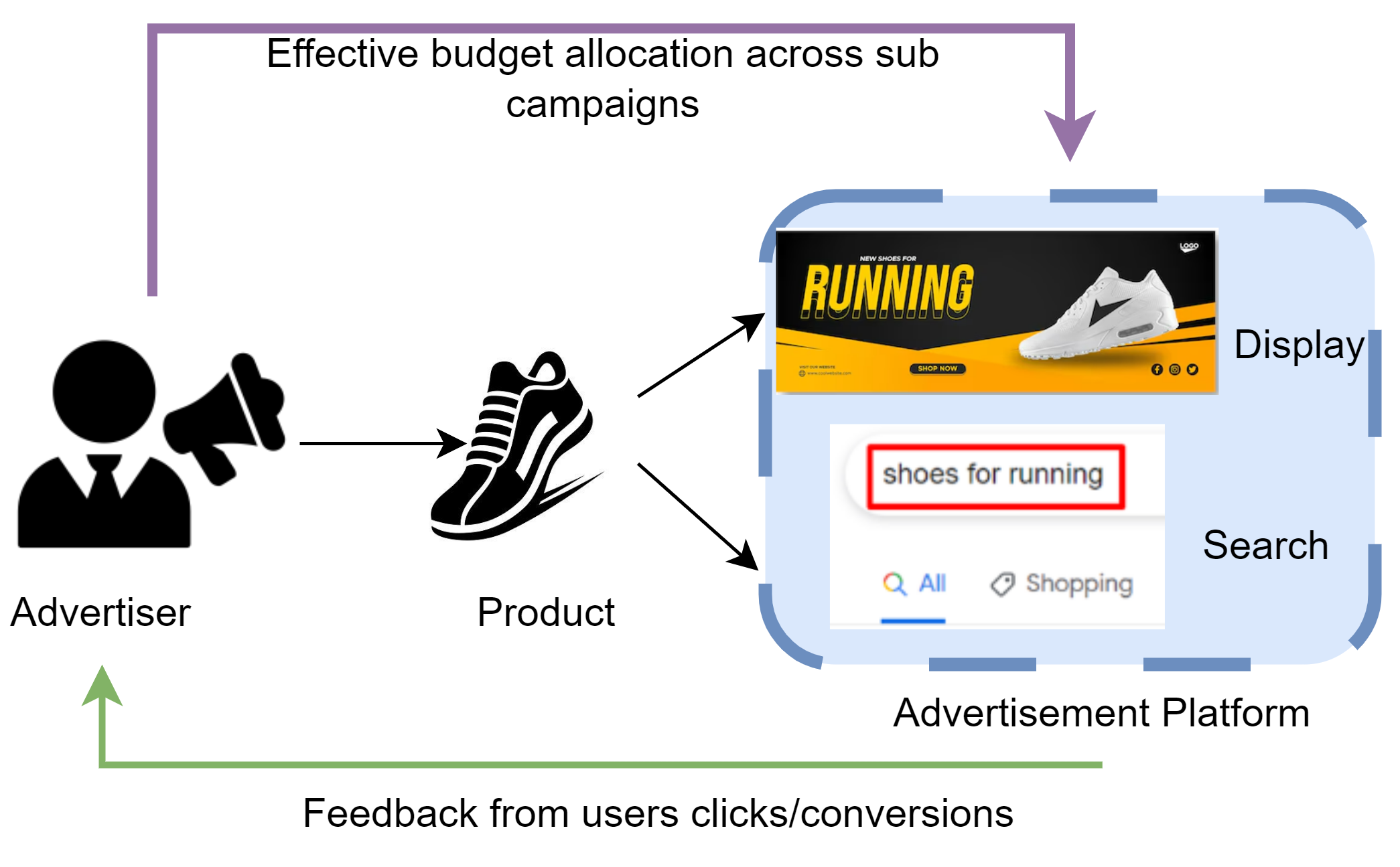}
\caption{Budget allocation across multiple sub campaigns in digital advertisement}
\label{budgetallocation}
\end{figure}

Research on budget allocation algorithms remain limited, despite its importance for advertisers. A well-planned spending strategy is crucial, as campaigns with inadequate budgets may struggle to reach high-quality traffic. Effective budget allocation can significantly boost Return on Ad Spend (ROAS) by ensuring that ads are displayed where users are most likely to engage \cite{nguyenpractical}. Multi-armed bandit strategies \cite{avadhanula2021stochastic, Nuara2018ACA, 10.5555/3618408.3618709} have proven highly effective for budget allocation due to their simplicity, ease of analysis, and practical implementation in real-world systems. However, these algorithms often suffer from inefficient exploration and may struggle to adapt to the evolving behavior of campaigns over extended periods. Non-stationarity is a frequent issue in online advertising environments \cite{italia2017internet}, where detecting changes and quickly adapting to them is critical.

A significant bottleneck in studying this important problem is the lack of rich open-source datasets and robust simulation environments. Business data is often proprietary, and the datasets used in previous research are typically not publicly accessible \cite{han2020contextual, 8545777, deng2023cross}, making it difficult to reproduce algorithmic results or build upon prior work. Moreover, directly testing budget allocation algorithms on real-world traffic can be both expensive and risky \cite{nguyen2023practical}. 

In this work we enhance the current budget allocation research with the following contributions:
\begin{itemize}
    \item We design a simulation environment that is capable of simulating logged data exhibiting characteristics of multichannel ad campaigns running over multiple months. To the best of our knowledge this is the first environment simulating non-stationary multichannel ad campaigns. The environment and data sets are publicly available \footnote{code contribution: \href{https://github.com/sony/ABA}{https://github.com/sony/ABA}} facilitating further exploration and development in this area.
    \item We enhance the combinatorial bandit budget allocation strategy with modified mean function and a novel exploration utility. The exploration utility accounts for campaign efficiency and filters target regions based on domain knowledge resulting in faster adaptation for long running non stationary campaigns. We also incorporate change point detection to adapt to changing market conditions.
    \item We theoretically show that the proposed method has sub-linear regret that is upper bounded by O($\sqrt T$) where T is the time horizon and reduces regret compared to standard exploration techniques. We empirically evaluate the proposed method on multiple real campaign data exhibiting higher reward, efficiency and lower regret compared with current SOTA baselines.
\end{itemize}
The paper is structured as follows: Section \ref{relatedwork} reviews related work, Section \ref{problem} introduces the problem formulation, Section \ref{prelim} covers the preliminaries, Section \ref{environment} presents the simulation environment, Section \ref{ABA} discusses the algorithm and provides a theoretical analysis, and Section \ref{empirical} reports the empirical results.


\section{Related Work}
\label{relatedwork}
Budget allocation across multiple ad campaigns \cite{farris2015marketing, 10.5555/3618408.3618709} has been extensively studied in industrial research by companies like Criteo \cite{DElia2019}, Netflix \cite{lewis2022incrementality}, and Lyft \cite{hancontextual}. A common approach is to discretize the budget and model each sub-campaign as an arm in a multi-armed bandit problem. The optimal allocation is obtained by solving a combinatorial optimization problem \cite{zhang2017multi}  based on the expected reward of each arm. In previous literature, domain knowledge has been used to formulate parametric models of the arms, approximating the cost-to-reward function with a power law \cite{hanbudget} or a sigmoid \cite{giglimulti}, followed by Thompson Sampling to handle uncertainty and induce exploration. However, these methods often overlook noise in the data, a critical factor in real-world deployments. In the presence of noise, parametric models can significantly deviate from the true reward function. A more flexible alternative is to model the reward function using Gaussian Process (GP) models \cite{nuara2022online, Nuara2018ACA}, which allow for greater adaptability. These algorithms typically use Upper Confidence Bound (UCB) or Thompson Sampling (TS) to guide exploration. However, unlike our approach, they do not incorporate domain knowledge to promote exploration, which can lead to higher regret. Additionally , these algorithms are mostly studied for budget allocation for a single day or month \cite{nuara2022online} which does not account for changing behaviours of the reward function, a characteristic often observed in campaigns running over many months.

Handling non-stationarity in multi-armed bandits is a well-studied problem in the literature \cite{cavenaghi2021non, besbes2014stochastic, re2021exploiting}. Common methods include passive approaches, such as sliding windows with UCB or TS sampling \cite{trovo2020sliding}, or using discounted rewards \cite{garivier2011upper}. Active methods, such as change point detection \cite{liu2018change, cao2019nearly}, offer a more dynamic approach. Passive methods either discard older data points or assign them less weight. However, in long-running campaigns where non-stationarity changes occur infrequently, these approaches are less effective. For our algorithm, we adopt an active approach to better handle reward function shifts.

\section{Problem Formulation}
\label{problem}
We follow the standard formulation of the Automatic Budget Allocation (ABA) problem from the literature \cite{Nuara2018ACA}. Consider an advertising campaign $\mathcal{A} = \{A_1, \dots, A_n\}$ with $N \in \mathbb{N}$, where each $A_j$ represents a sub-campaign in the portfolio. The campaigns run over a finite time horizon of $T \in \mathbb{N}$ days with a budget $\mathcal{B} = \{\bar{b}_1, \ldots, \bar{b}_T\}$, where $\bar{b}t \in \mathbb{R}^+$ denotes the maximum budget that can be spent at time $t \in {1, \ldots, T}$. For each day and sub-campaign $A_j$, the advertiser must allocate a budget $b_{j,t} \in \left[ \underline{b}_{t}, \bar{b}_{t} \right]$, where $\underline{b}_{t} \in \mathbb{R}^+$ represents the minimum budget. After setting the budget, the platform determines the cost $x_{j,t}$, and the advertiser receives feedback in the form of rewards (such as clicks or conversions) from an unknown function $n_{j,t}$. The goal of the advertiser is to determine the optimal budget allocation across all sub-campaigns to maximize the cumulative return on investment. Formally, the problem is formulated as the following constrained optimization problem:

\begin{align}
    \max_{x_{j,t}} &\quad \sum_{j=1}^{N} n_{j,t}(x_{j,t}) \tag{1a} \\
    \text{s.t.} \quad & \sum_{j=1}^{N} b_{j,t} \leq \bar{b}_t \tag{1b} \\
    & \underline{b}_{t} \leq b_{j,t} \leq \bar{b}_{t} \quad \forall\, j \tag{1c} \\
\end{align}

Here, $x_{j,t}$ represents the cost spent on the sub-campaign $A_j$ at time $t$. The cost-to-reward relationship $n_{j,t}$ is dynamic, often changing over time due to market fluctuations. In particular, we focus on settings where the reward function changes abruptly, modeled as a piece-wise constant function of time that shifts a finite number of times. Formally, in the non-stationary setting, a \textit{break-point} $p \in {1, \dots, T}$ is defined as a round where the expected reward with respect to budget set $B$ of at least one sub-campaign undergoes a change, i.e.,

\begin{align}
\mathbb{E}[\sum_{i=0}^{B} n_{j,p-1}(b_{i})] \neq \mathbb{E}[\sum_{i=0}^{B} n_{j,p}(b_{i})] \quad \text{for some sub-campaign} \; j.
\end{align}

Let $\mathcal{P} = {p_1, \dots, p_{T}}$ denote the set of breakpoints, with $p_0 = 1$, partitioning the rounds into a set of \textit{phases} ${\mathcal{F}_1, \dots, \mathcal{F}_T}$, where each phase is defined as:

\begin{align}
\mathcal{F}_\phi = \{ t \in \{1, \dots, T\} \mid p_{\phi-1} \leq t < p_\phi \}.
\end{align}

Within each phase $\mathcal{F}_\phi$, the reward function for sub-campaign $A_j$ remains constant and is given by:

\[
\mu_{j,\phi} = \mathbb{E}[\sum_{i=0}^{B} n_{j, \phi}(b_{i})] \quad \text{for } t \in \mathcal{F}_\phi.
\]

To effectively detect abrupt changes in the reward functions, we follow two standard assumptions commonly used in non-stationary multi-armed bandit (MAB) settings \cite{10.1007/978-3-030-86486-6_4}:

\textbf{Assumption 1} $\exists \, \tau \in \mathbb{R}^+$, \textit{known to the learner, such that for each sub campaign} $A_j$ \textit{whose expected reward changes between consecutive phases} $\phi$ \textit{and} $\phi + 1$, \textit{we have:}

\[
|\mu_{j,\phi} - \mu_{j,\phi+1}| \geq \tau.
\]

This lets the learner decide on a minimum possible magnitude of change such that the learner is able to detect it.

\textbf{Assumption 2} \textit{There exists a time period} $T_p$, \textit{unknown to the learner, such that:}

\[
\min_{\phi \in \{1, \dots, T\}} (p_\phi - p_{\phi-1}) \geq T_p.
\]

This prevents the breakpoints from being too-close to one another.

\textbf{Assumption 3} \label{ass2} Based on previous literature, the reward function at any phase  $n_j(x)$ exhibits the following properties \cite{gigli2024multi, han2020contextual}:
\begin{enumerate}
    \item $n_j(x)$ is continuous and smooth to at least the second order.
    \item $n_j(x)$ is monotonically increasing with the cost (more spend always yields more clicks/conversions), i.e., $n_j'(x) > 0$.
    \item $n_j(x)$ has a diminishing marginal impact, i.e., $n_j''(x) < 0$.
\end{enumerate}

\begin{figure*}[!h]
\centering
\includegraphics[width=\textwidth]{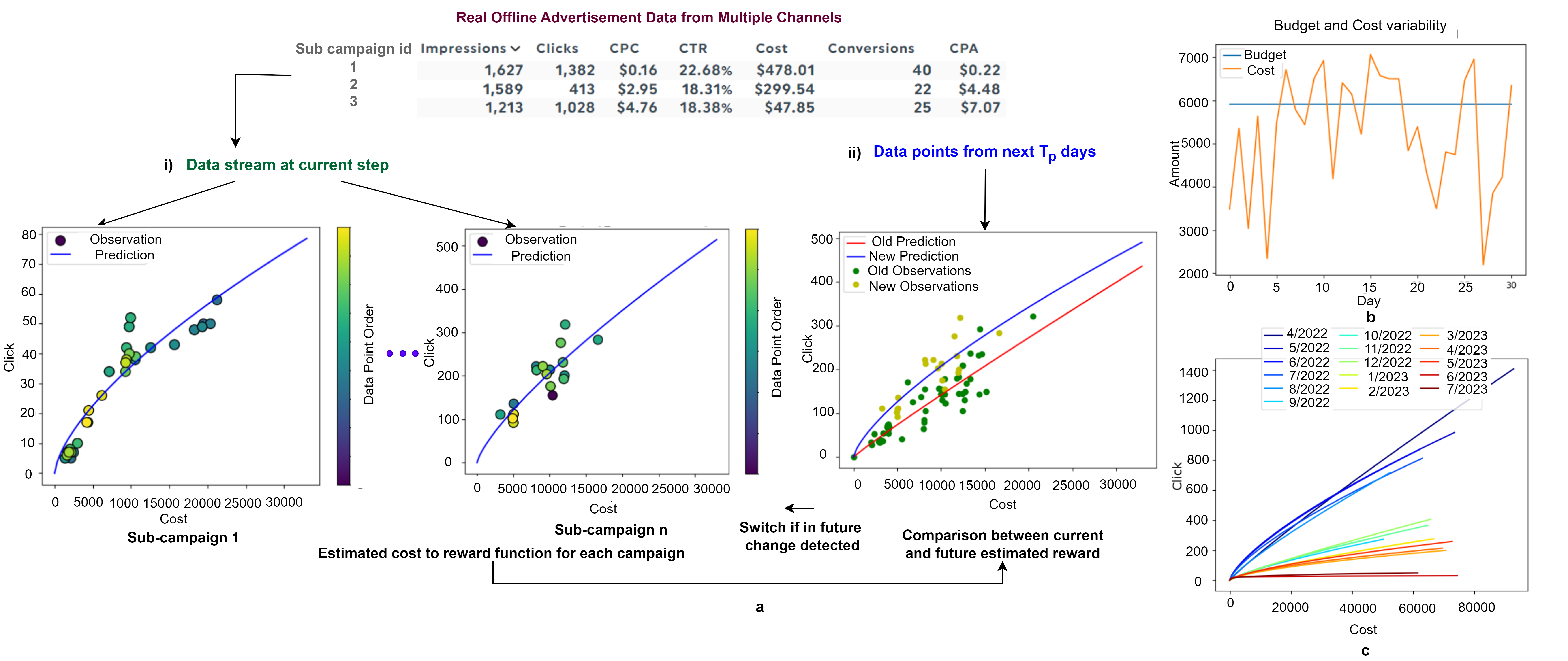}
\caption{a) Architecture of the simulation environment where the reward function learned from the logged data b) Variability of budget to cost consumption in the environment c) Changing reward functions over different months in the environment}
\label{simenv}
\end{figure*}

\section{Preliminaries}
\label{prelim}
In a combinatorial semi-bandit framework \cite{pmlr-v28-chen13a}, the agent selects a subset of options, referred to as super-arms, from a finite set of available choices, known as arms. This selection is subject to combinatorial constraints, such as the knapsack constraint. In this work, the reward of each arm is modeled using Gaussian Process Regression, and the optimization is solved using a multi-choice knapsack algorithm. We briefly explain each of these concepts as follows:

\subsection{Gaussian Process Regression}

Gaussian Process Regression (GPR) \cite{schulz2018tutorial} is employed to model the relationship between budget allocation and resulting reward. GPR is a non-parametric, probabilistic method that provides both predictive mean and uncertainty estimates for a given set of inputs. Formally, a GP is defined as:

\[
f(\mathbf{x}) \sim \mathcal{GP}(\mu(\mathbf{x}), k(\mathbf{x}, \mathbf{x}'))
\]

where \( f(\mathbf{x}) \) represents the unknown function that relates the input variables \( \mathbf{x} \) (e.g., budget) to the output variables (e.g., clicks). The mean function \( \mu(\mathbf{x}) \) is typically assumed to be zero. The covariance or kernel function \( k(\mathbf{x}, \mathbf{x}') \) encodes the correlation between any two input points.

The predictive mean \( \mu(\mathbf{x}_*) \) and variance \( \sigma^2(\mathbf{x}_*) \) at a test point \( \mathbf{x}_* \) are given by:

\[
\mu(\mathbf{x}_*) = \mathbf{k}_*^\top (\mathbf{K} + \sigma_n^2 \mathbf{I})^{-1} \mathbf{y}
\]
\[
\sigma^2(\mathbf{x}_*) = k(\mathbf{x}_*, \mathbf{x}_*) - \mathbf{k}_*^\top (\mathbf{K} + \sigma_n^2 \mathbf{I})^{-1} \mathbf{k}_*
\]

where \( \mathbf{k}_* = [k(\mathbf{x}_*, \mathbf{x}_1), \dots, k(\mathbf{x}_*, \mathbf{x}_n)]^\top \) is the vector of covariances between the test point \( \mathbf{x}_* \) and each training input \( \mathbf{x}_i \), and \( \mathbf{K} \) is the covariance matrix computed over the training inputs, with entries \( K_{ij} = k(\mathbf{x}_i, \mathbf{x}_j) \) and y is the observed mean. The term \( \sigma_n^2 \) represents the variance of the noise in the observations.

The budget-to-reward relationship is modeled using the Radial Basis Function (RBF) kernel. The RBF kernel assumes a smooth and continuous relationship, defined as \( k_{\text{RBF}}(\mathbf{x}, \mathbf{x}') = \sigma_f^2 \exp\left( -\frac{||\mathbf{x} - \mathbf{x}'||^2}{2 l^2} \right) \), where \( \sigma_f^2 \) is the signal variance and \( l \) is the length scale.

\subsection{Multi Choice Knapsack}

The optimization problem can be cast as a modified version of the knapsack problem from \cite{kellerer2004multiple} called Multi Choice Knapsack (MCK). Given an estimated reward model of each sub-campaign and an evenly spaced discritization of the daily budget $B \subset \mathcal{B}$, the optimal reward for each sub-campaign can be identified through enumeration. The solution can be efficiently computed with a dynamic programming approach. The matrix $M(j,b)$ with $j \in {1 \dots N}$ and $b \in B$. For a particular $\mathcal{F}_\phi$, The matrix is iteratively filled: each element is initialized as \( M(j, b) = 0 \) for all \( j \) and \( b \in \mathcal{B} \). For \( j = 1 \), the value is set: 

\[
M(1, b) = n_1(b) \quad \forall b \in B
\]

This equation represents the best budget allocation for the sub-campaign \( A_1 \) if it were the only sub-campaign to consider. For \( j > 1 \), each matrix entry is updated as follows:

\[
M(j, b) = \max_{b' \in B, b' \leq b} \left( M(j - 1, b') + n_j(b - b') \right)
\]

Then the maximum value among all combinations is selected.

At the end of the recursion, the optimal solution is found by evaluating the matrix cell corresponding to:

\[
\max_{b \in B} M(N, b)
\]

To retrieve the corresponding budget allocation, the matrix is traced back to store the partial assignments that maximize the total value. The time complexity of this algorithm is \( O(N H^2) \), where \( N \) is the number of subcampaigns and \( H = |\mathcal{B}| \) represents the cardinality of the budget set.

\section{Simulation Environment}
\label{environment}
A major challenge in studying budget allocation algorithms for digital ads is the lack of open-source simulation environment capable of simulating logged offline data. Previous studies have either relied on synthetic data \cite{Nuara2018ACA, 10.5555/3618408.3618709}, which fails to fully capture real-world dynamics, or on proprietary data that is not publicly available \cite{han2020contextual, 8545777}, rendering research results difficult to reproduce. Available real world datasets like criterio dataset \cite{diemert2017attribution} do not provide structured campaign groups and are limited to a time horizon of 30 days. To bridge this gap, we designed a simulation environment that mimics the behaviour of long running ad campaigns from logged data. The simulation environment and the logged data will be released publicly to facilitate reproducible research. The architecture of the simulation environment is depicted in Fig \ref{simenv}a.

The daily budget is set as per the total monthly cost consumed by all the campaigns of a campaign group divided by the number of days per month. In any realistic ad delivery platform the actual spent cost $x_{j,t}$ is in not equivalent to the allocated budget and depends on the platforms internal learning algorithms. For example the Google Ads platform provides the following guidelines \cite{googleAboutAverage} :
1) The spent amount can be lower or 2 times higher than the daily budget on any particular day.
2) The total spent budget is not more than 30.4 the average daily budget. We model this variability in daily budget spent using a truncated normal distribution:
\begin{align}
x_{j,t} \sim \mathcal{N}(b_{j,t}, \sigma^2) \;\; s.t \;\; 0 \leq x_{j,t} \leq 2*b_{j,t}
\end{align}
The cost variability is shown in Fig \ref{simenv}b. Following \cite{han2020contextual} we model the cost to reward function of each sub campaign as a power law function with noise.

\begin{align}
n_j(x_{j,t}) = {\alpha_c}*x_{j,t}^{\omega_c} + \epsilon
\end{align}

Where $\epsilon$ adds a small error in observation. The simulation environment updates the reward model every day with data points from the logged data of that day. The parameters $\alpha_c$ and $\omega_c$ are estimated from data using curve fitting as shown in Fig \ref{simenv}a i). In order to model abrupt changes between the reward functions we maintain a power law model $\alpha_f$ and $\omega_f$ for the next $T_p$ days from the current time point in simulation (Assuming a stationary period of length $T_p$) of data. If a change is detected, i.e., when $\alpha_c$ and $\alpha_f$ differ more than $20\%$, the current model is replaced with the future model on the onset of detected change as shown in Fig \ref{simenv}a ii). This allows the function to change at arbitrary points during the run of the campaign as would happen in a real campaign as depicted in Fig \ref{simenv}c.

\begin{figure*}[!t]
\centering
\includegraphics[width=0.9\textwidth]{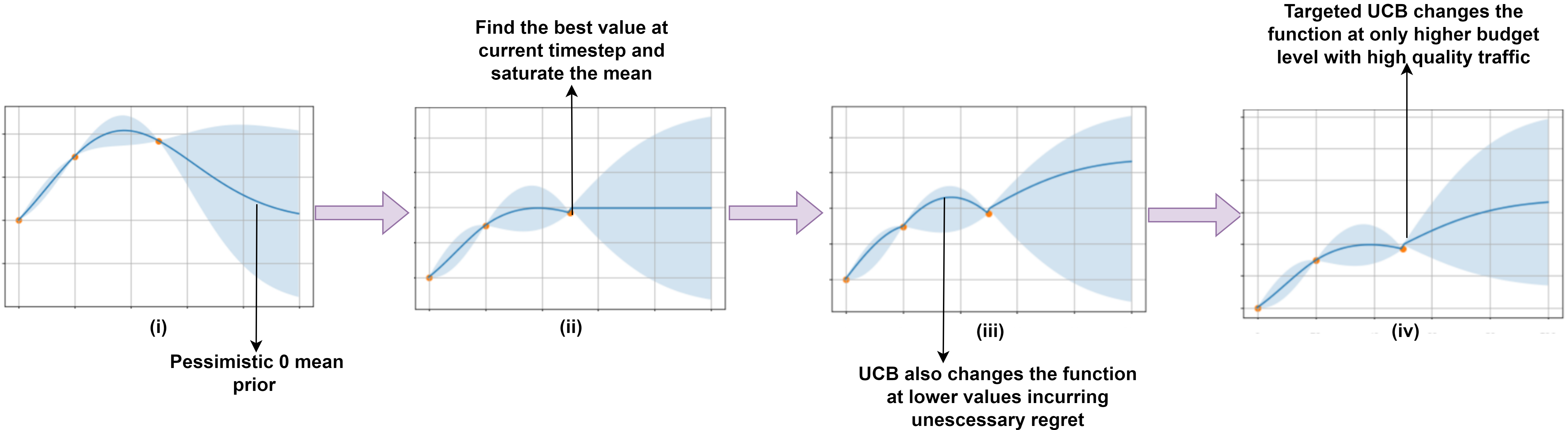}
\caption{A simple representation of the GP estimation with saturated mean and targeted UCB exploration}
\label{gpexample}
\end{figure*}

\section{Automatic Budget Allocation Algorithm}
\label{ABA}
The ABA algorithm is summarised in Algo \ref{alg:TUCB-MAE} which involves the following broad steps: 
\begin{enumerate}
\item Estimation of reward function using GP 
\item Predicting rewards for each arm of the bandit 
\item Allocating budget using multi-choice knapsack 
\item Change point detection. 
\end{enumerate}
The algorithm enhances the automatic budget allocation strategy to cater to practical considerations. In any multichannel advertising application exploration is expensive. This means we should be selective about spending budget in regions where we expect higher gains. First we observe that a zero-mean Gaussian Process Regressor as used in \cite{nuara2022online} obtains a pessimistic prior over the budget range as depicted in Fig \ref{gpexample} (i). This prior restricts effective exploration to higher ranges of budget where quality traffic might be present.  To address this, we modify the mean of the GP model with a saturating mean function for each sub-campaign $j$ as follows:

\begin{align}
\label{satmean}
\hat{n}_j =
\begin{cases}
\hat{n}_{j \text{max}}, & \text{if } b_{j,i} > b_{j \text{max}} \\
\hat{n}_j, & \text{otherwise}
\end{cases}
\end{align}

Where $\hat{n}_j$ is the GP estimate of $n_j$ with time subscript removed for brevity and $b_{j \text{max}}$ is the current budget level with highest reward value for campaign j and $i \in B$. This allows the mean to saturate at the last estimated maximum observed reward for a campaign as shown in Fig \ref{gpexample} (ii). Next, we introduce a modified Upper Confidence Bound exploration strategy to enhance the performance of the combinatorial bandit approach. The modified exploration strategy is defined as follows:

\begin{align}
\label{moducb}
\tilde{n}_j(\cdot) \gets \hat{n}_j(\cdot) + \{\beta * (1 - \theta_j) * \sigma_j\}|\mathbb{I}_{b_{j,i} > b_{j, max}}
\end{align}

Where $\beta$ is the exploration factor for balancing exploration and exploitation. The proposed modified UCB promotes the following:
\begin{itemize}
\item $\theta_j$ represents the efficiency of arm j. For example, Cost per Click (CPC) can be used as $\theta_j$ when maximizing clicks where $\theta_j = cpc_j$. $cpc_j = (\sum_t \frac{cost_{j,t}}{click_{j,t}})/max(\sum_t \frac{cost_{j,t}}{click_{j,t}})$ is the normalized cost per click of sub campaign $j$. A lower cpc denotes higher efficiency of the sub-campaign. The inclusion of term $1 - cpc_j$ incentivizes the policy to perform aggressive explorations for efficient arms. This term can be replaced by any other metric of efficiency as per advertiser's objective. For example, the Cost per Acquisition (CPA) can be chosen as the exploration incentive during maximizing conversions. 
\item The term $\mathbb{I}_{b_{j,i} > b_{j, max}}$ denotes and indicator function that checks whether a discritized budget level used by MCK is higher than the current observed budget level having highest predicted reward. The uncertainty based exploration is only targeted towards regions that contain more information than the current best knowledge as illustrated in Fig \ref{gpexample} (iv). Without this targeted exploration the algorithm may incur unnecessary regret by exploring lower budget levels as shown in Fig \ref{gpexample} (iii).
\end{itemize}

\begin{algorithm}
\caption{TUCB-MAE}\label{alg:TUCB-MAE}
\begin{algorithmic}[1]
\Require Set $B$ of discretized budget values, Initial Old Model $\{\mathcal{M}_j^{(0)}\}_{j=1}^N$, Current Model $\{\tilde{\mathcal{M}}_j^{(0)}\}_{j=1}^N$, Daily Budget limit $\bar{b}_t$, time horizon $T$, Memory $buffer$
\For{$t \in \{1, \dots, T\}$}
    \For{$j \in \{1, \dots, N\}$}
        \If{$t = 1$}
            \State $\mathcal{M}_j \gets \mathcal{M}_j^{(0)}$
            \State $\tilde{\mathcal{M}_j} \gets \tilde{\mathcal{M}}_j^{(0)}$
        \Else
            \State $ y_{j,t} = {n}_{j}(x_{j,t})$
            \State $buffer.append(y_{j,t}, x_{j,t})$
            \State $\mathcal{M}_j \gets \text{Update}\left(\mathcal{M}_j, buffer)\right)$
            \State $\tilde{\mathcal{M}_j} \gets \text{Update}\left(\tilde{\mathcal{M}_j}, buffer[:window_{length}])\right)$
        \EndIf
        \State Check Eq \ref{cpd} > $\tau$ to detect breakpoint
        \If{breakpoint}
        \State $buffer \gets buffer[:window_{length}]$
        \EndIf
        \State Use Eq \ref{satmean} to saturate mean
        \State Use Eq \ref{moducb} to select the next exploration points
    \EndFor
    \State $\{(\hat{x}_{j,t})\}_{j=1}^N \gets \text{Optimize}\left(\{(\tilde{n}_{j}(\cdot), B)\}_{j=1}^N, \bar{b}_t\right)$
    \State Pull $(\hat{x}_{1,t}, \dots, \hat{x}_{N,t})$
\EndFor
\end{algorithmic}
\end{algorithm}

For the non stationary change detection we maintain two models. $\mathcal{M}_j$ denotes the model which estimates the reward function for data points of phase $\mathcal{F}_\phi$ until break-point $p_{\phi+1}$ is detected. $\tilde{\mathcal{M}}_j$ denotes the model using data points from current $window_{length}$. We then perform change point detection using a Mean Average Error test over the entire budget set to check if the predictions from the models have changed beyond a threshold $\tau$.
\begin{align}
\label{cpd}
pred_{diff} = \frac{1}{B} \sum_{i=1}^{B} \mathcal{M}_j (b_i) - \tilde{\mathcal{M}}(b_i) 
\end{align}
MAE is used due to its ease of implementation for practical usage. Any sophisticated change point detection strategy can be used in place of MAE. If a change is detected the data buffer is refreshed with the current $window_{length}$ data denoting the start of a new phase $\mathcal{F}_{\phi+1}$. 

We now theoretically analyze the regret bound of the proposed method and show that the regret bound reduces for the proposed UCB utility under Assumption \ref{ass2}.

\begin{lemma}[From \cite{10.5555/3104322.3104451}]
\label{lemma1}
Given the realization of a GP $f(\cdot)$, the estimates of the mean $\hat{\mu}_{t-1}(b)$ and variance $\hat{\sigma}^2_{t-1}(b)$ for the input $b$ belonging to the input space $B$, for each $\beta \in \mathbb{R}^+$ the following condition holds:
\[
\mathbb{P} \left( \left| f(b) - \hat{\mu}_{t-1}(b) \right| \geq \sqrt{\beta} \, \hat{\sigma}_{t-1}(b) \right) \leq e^{-\frac{\beta}{2}},
\]
for each $b \in B$.
\end{lemma}

\begin{table*}[!h]
\small
\centering
\caption{Comparison of proposed algorithm with SOTA baselines using logged campaigns for real products running on different ad delivery platforms reported for random seeds 1, 42, and 76. The reported values have been divided by 1000. Each row in the table represents the cumulative Clicks$\uparrow$, Regret$\downarrow$ and CPC ({¥}) $\downarrow$ of each method.}
\label{clickresult}
\begin{tblr}{
  width = \linewidth,
  colspec = {Q[90]Q[75]Q[75]Q[75]Q[100]Q[100]Q[100]Q[100]Q[100]Q[100]},
  cell{1}{1} = {c},
  cell{1}{2} = {c},
  cell{1}{4} = {c},
  cell{1}{5} = {c},
  cell{1}{6} = {c},
  cell{1}{7} = {c},
  cell{1}{8} = {c},
  cell{1}{10} = {c},
  hline{1-2,7} = {-}{},
  hline{3-6} = {1-10}{},
}
\textbf{Product Type} & \textbf{Sub campaign Groups} & \textbf{Duration} & \textbf{Metric} & \textbf{TUCB-MAE (Ours)} & \textbf{UCB-MAE} & \textbf{UCB-NCPD} & \textbf{UCB-SW} & \textbf{TS-SW} & \textbf{UCB-DS}\\
{Attendance System\\Platform A} & {Search-1\\Search-2\\Display} & {01-07-22 -~\\30-07-23} & {Clicks $\uparrow$ \\ Regret $\downarrow$ \\ CPC $\downarrow$} & {\textbf{55.62  $\pm$  1.35}\\\textbf{14.97  $\pm$ 1.39}\\\textbf{52.94   $\pm$ 0.83}} & {44.52  $\pm$ 1.31\\25.39  $\pm$ ~1.40\\63.80  $\pm$ ~2.54} & {53.10  $\pm$ ~0.60\\17.59  $\pm$ ~0.74~\\53.43  $\pm$ ~0.36} & {47.27  $\pm$ ~2.06\\22.65  $\pm$ ~2.05\\59.64  $\pm$ ~1.64} & {49.34   $\pm$ 2.63\\20.60  $\pm$ ~2.57\\58.47  $\pm$ ~3.01} & {29.82  $\pm$ ~2.88\\40.02  $\pm$ ~2.68\\85.72  $\pm$ ~1.29}\\
{Predictive Analysis Tool\\Platform A} & {Search\\Display\\Discovery} & {01-04-22 -\\10-09-23} & {Clicks $\uparrow$ \\ Regret $\downarrow$ \\ CPC $\downarrow$} & {\textbf{243.11$\pm$6.78}\\\textbf{66.71   $\pm$ 5.36}\\\textbf{32.31   $\pm$ 1.87}} & {218.84   $\pm$ 7.36\\78.45   $\pm$ 6.74\\47.62   $\pm$ 1.47} & {220.98   $\pm$ 6.77\\76.86   $\pm$ 6.72\\ 46.62   $\pm$ 1.46} & {217.44   $\pm$ 7.24\\80.95   $\pm$ 7.43\\46.36   $\pm$ 1.40} & {138.75   $\pm$ 0.91\\156.90   $\pm$ 0.83\\119.27   $\pm$ 1.49} & {187.71   $\pm$ 7.74\\1099.52$\pm$7.21\\67.81   $\pm$ 1.27}\\
{Internet Service Provider\\Platform A} & {Search\\Display\\Discovery} & {01-04-22 -~\\19-10-22} & {Clicks $\uparrow$ \\ Regret $\downarrow$ \\ CPC $\downarrow$} & {\textbf{4.90   $\pm$ 0.09}\\\textbf{226.49$\pm$0.08}\\\textbf{35.93  $\pm$ 0.20}} & {4.52   $\pm$ 0.03\\227.43  $\pm$  0.04\\37.30  $\pm$  0.22} & {4.72   $\pm$ 0.04\\226.63   $\pm$ 0.03\\36.41   $\pm$ 0.32} & {4.75   $\pm$ 0.09\\226.61   $\pm$ 0.07\\36.18   $\pm$ 0.54} & {4.65   $\pm$ 0.21\\226.70   $\pm$ 0.20\\37.23   $\pm$ 1.09} & {3.92  $\pm$  0.06\\227.41  $\pm$  0.06\\42.00  $\pm$  0.32}\\
{Product 17276\\Platform B} & 5 Display & {01-10-23\\-01-07-24} & {Clicks $\uparrow$ \\ Regret $\downarrow$ \\ CPC $\downarrow$} & {\textbf{227.79$\pm$3.18}\\\textbf{35.38   $\pm$ 3.13}\\\textbf{9.47   $\pm$ 0.07}} & {214.42   $\pm$ 4.46\\48.45   $\pm$ 4.43\\10.00   $\pm$ 0.08} & {215.33   $\pm$ 2.51\\47.51   $\pm$ 2.56\\9.92   $\pm$ 0.03} & {214.36   $\pm$ 1.72\\48.53   $\pm$ 1.70\\9.91   $\pm$ 0.06} & {223.01   $\pm$ 3.38\\39.99   $\pm$ 3.20\\9.63   $\pm$ ~0.03} & {201.57  $\pm$  2.38\\61.44  $\pm$  2.43\\10.26  $\pm$  0.06}\\
{Product 15981\\Platform B} & 4 Display & {01-10-23\\- 01-07-24} & {Clicks $\uparrow$ \\ Regret $\downarrow$ \\ CPC $\downarrow$} & {\textbf{105.92$\pm$3.06}\\\textbf{20.35   $\pm$ 2.87}\\\textbf{21.93   $\pm$ 0.20}} & {93.97   $\pm$ 9.59\\31.64   $\pm$ 9.31\\26.81   $\pm$ 2.79} & {82.22   $\pm$ 10.64\\43.15  $\pm$  10.57\\29.20   $\pm$  2.57} & {92.79   $\pm$ 10.41\\32.91   $\pm$ 9.93\\27.07   $\pm$ 2.51} & {92.45   $\pm$ 7.80\\32.96   $\pm$ 7.74\\27.33   $\pm$ 2.91} & {73.23   $\pm$ 10.49\\52.09   $\pm$ 10.45\\31.91   $\pm$ 2.83}\\
\end{tblr}
\end{table*}

\begin{proposition}
Let us consider an ABA problem over T rounds where the function $\hat{n}_j(b)$ is the realization of a GP, using TUCB-MAE algorithm with the following upper bound on the reward function $\hat{n}_j(b)$:
\[
u^{(n)}_{j,t-1}(b) := \hat{\mu}_{j,t-1}(b) + \sqrt{\beta_{j,t}} \hat{\sigma}_{j,t-1}(b)
\]
where $b$ is a budget level,n denotes the round and j is the campaign, with probability at least $1 - \delta$, it holds:
\[
\mathcal{R}_T(U) = \tilde{\mathcal{O}} \left( \sqrt{TN \sum_{j=1}^{N} \gamma_{T}(\hat{n}_j)} \right),
\]
where the notation $\tilde{\mathcal{O}}\left( \cdot \right)$ disregards the logarithmic factors.
\\
\\
\textit{Proof Sketch: } 
It can be derived regret is lower bounded by $\hat{\sigma}_{j,t-1}(a)$ where $a$ is the action with max $\hat{\sigma}_{j,t-1}$ for campaign $j$. Using Lemma 5.6 of \cite{10.5555/3104322.3104451},  the information gain provided by the observations \( n_{t-1} = (\tilde{n}_{j,1}, \dots, \tilde{n}_{j,t-1}) \) corresponding to the actions \( (a_{j,1}, \dots, a_{j,t-1}) \) is:
\[
IG(\hat{n}_{t-1} | \hat{n}_{j}) = \frac{1}{2} \sum_{h=1}^{t-1} \log \left( 1 + \frac{\hat{\sigma}^2_{j,h}(a_j,h)}{\lambda} \right).
\]
and $\hat{\sigma}_{j,t-1}(a)$ can be bounded by:
\[
\sigma^2_{j,h}(a_j,h) \leq \frac{\log \left( 1 + \frac{\hat{\sigma}^2_{j,h}(a_j,h)}{\lambda} \right)}{\log \left( 1 + \frac{1}{\lambda} \right)}
\]
and regret can be derived as a lower bound of IG, 
\\
with $\beta_{j,t} = 2 \log \left( \frac{\pi^2 NMt^2}{3 \delta} \right)k_j$, $k_j = (1-\theta_j)$. For every $\delta \in (0,1)$ the following holds with probability at least $1 - \delta$ (using Lemma \ref{lemma1}),

\[
\mathcal{R}_T(U) \leq 4T\beta_{T} \left\{ \frac{1}{\log \left( 1 + \frac{1}{\lambda} \right)} \sum_{j=1}^{N} \gamma_T(\hat{n}_j) \right\}
\]
where $\lambda$ is the variance of the measurement noise of the reward function $n_j(\cdot)$ and \( \gamma_T(\hat{n}_j) \) is the total information gain . 

Since the regret is bounded by information gain, if we explore values of \( b_{j,t} \leq b_{j\max,t} \), by monotonicity, we have:
\[
\hat{n}_j(b_{j,t}^*) \geq \hat{n}_j(b_{j\max,t}) \geq \hat{n}_j(b_{j,t}).
\]
Where $b_{j,t}^*$ is the budget level with maximum reward of arm j. This means that exploring in this region incurs unnecessary regret because we are not gaining new information about potentially better actions. By restricting exploration to values \( b_{j,t} > b_{j\max,t} \), the effective space of arms to explore is reduced. This reduces \( \gamma_T(\hat{n}_j) \), which in turn reduces the regret bound. Specifically, if we denote the restricted exploration space by \( X_j^+ \), we have:
\[
\gamma_T(\hat{n}_j, X_j^+) \leq \gamma_T(\hat{n}_j).
\]
Thus, under monotonocity assumption of $\hat{n}_j$
\[
R_T(U^+) = O \left( \sqrt{TN \sum_{j=1}^{N} \gamma_T(\hat{n}_j, X_j^+)} \right) \leq \mathcal{R}_T(U)
\]

detailed proof is given in supplementary material.

\end{proposition}

\begin{figure*}[!t]
\centering
\includegraphics[width=0.9\textwidth]{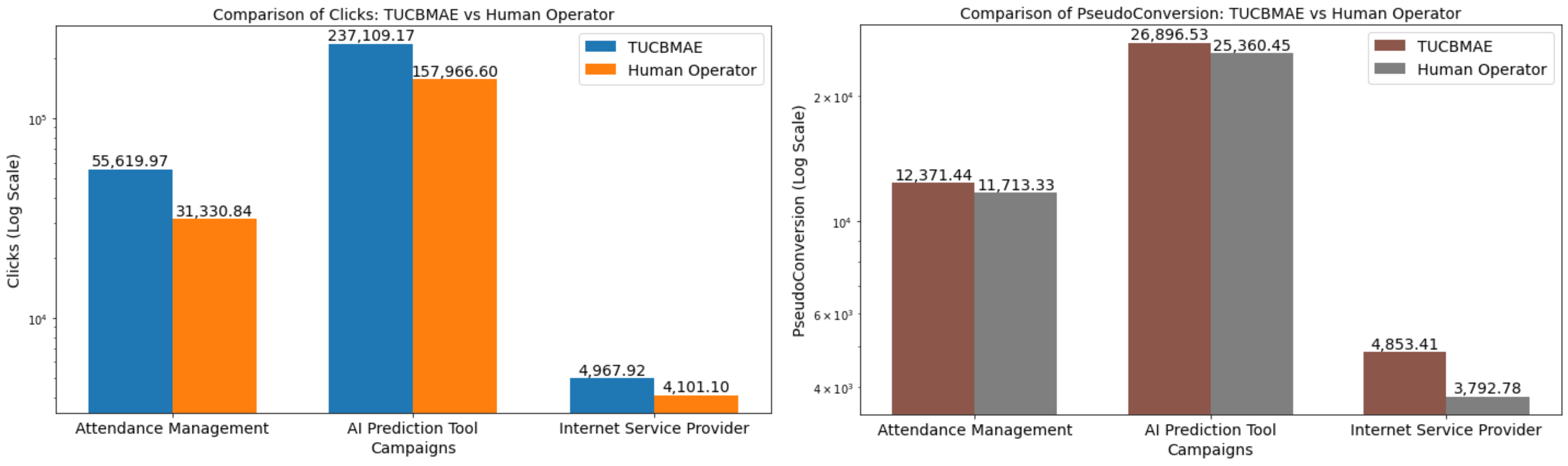}
\caption{Comparison with respect to the human operator's budget allocation from the logged dataset}
\label{clickps}
\end{figure*}


\section{Empirical Studies}
\label{empirical}
We perform empirical experiments on multiple real logged campaign data obtained from different platforms. We denote the different advertisement platforms as Platform A and Platform B. The hyper-parameter choices are reported in supplementary material. For experimental analysis we choose $T_p = 20$ assuming a stationary period of 20 days and $window_{length} = 7$ days. The budget discretization granularity is 500. We simulate these campaigns in the simulation environment allowing the experiments to be conducted for long running campaigns with changing behaviour due to market dynamics. The noise  ($\epsilon$) is sampled from a normal distribution $\mathcal{N}(0,0.1)$. The proposed algorithm is compared against the following SOTA baselines:

\begin{enumerate}
    \item UCB - MAE: Represents a combinatorial multi-arm bandit strategy with upper confidence bound for exploration and mean average error for change point detection. Represents the class of active approaches where the reward function is re-learned based on change point detection \cite{10.1007/978-3-030-86486-6_4}. Comparison shows superiority of our proposed exploration utility.
    \item UCB - NCPD: Is a combinatorial bandit strategy with UCB exploration and no change point detection depicting the importance of change point detection.
    \item UCB - SW \cite{garivier2011upper} : Represents a combinatorial bandit algorithm with UCB exploration and sliding window of fixed length (10 days) for non stationary adaption and same exploration parameter $\beta$ as our algorithm.
    \item TS-SW \cite{fiandri2024sliding}: Represents a combinatorial bandit algorithm with thompson sampling exploration and sliding window of fixed length (10 days) for non stationary adaption.
    \item UCB-DS \cite{garivier2011upper} : A combinatorial bandit strategy with discounting past data using a factor 0.9 and UCB exploration strategy.
\end{enumerate}

We report the results in Table \ref{clickresult} with respect to three metrics explained as follows:
\\
\textbf{Clicks:}  A higher number of clicks generally reflects increased user engagement, making it a key measure of effective budget allocation.
\\
\textbf{Regret:} We report the average cummulative regret compared to an oracle optimizer which has access to the parameters of the true reward function in the simulation environment.
\\
\textbf{Cost Per Click (CPC):} The average cost per click for all the sub-campaigns in a campaign group. A lower CPC denotes higher ROAS and efficiency for advertisers.

The algorithms are tested across different types of products with varied user base as reported in Table \ref{clickresult}. Each product contains of multiple sub-campaigns running together for more than 5 months. The sub-campaigns are distributed across multiple channels. Display advertisements visually engaging ads placed at different web-channels that a user visits. Search campaigns allows advertisements to be placed across a search engine's network of search results. Search-1 campaigns target users searching with specific product related keywords whereas Search-2 campaigns target a wider audience with generic keywords related to the domain of the product. The results in Table \ref{clickresult} exhibits the effectiveness of the proposed algorithm with higher clicks, lower regret and lower cpc for all products. We also note that discounted reward based adaptation strategy renders the lowest performance as providing lower weights to past observation refrains GP from adapting to the true function. In Fig \ref{rewardAakashi} we plot the reward over entire duration of campaign for  attendance management system. The plots show TUCBMAE algorithm achieves higher rewards than sliding window during stationary periods and can adapt to non stationary change faster than UCB algorithm. However, the algorithm unable to adapt to very short period of no-stationary changes as observed around day 250. 

\begin{figure}[!h]
\centering
\includegraphics[width=0.48\textwidth]{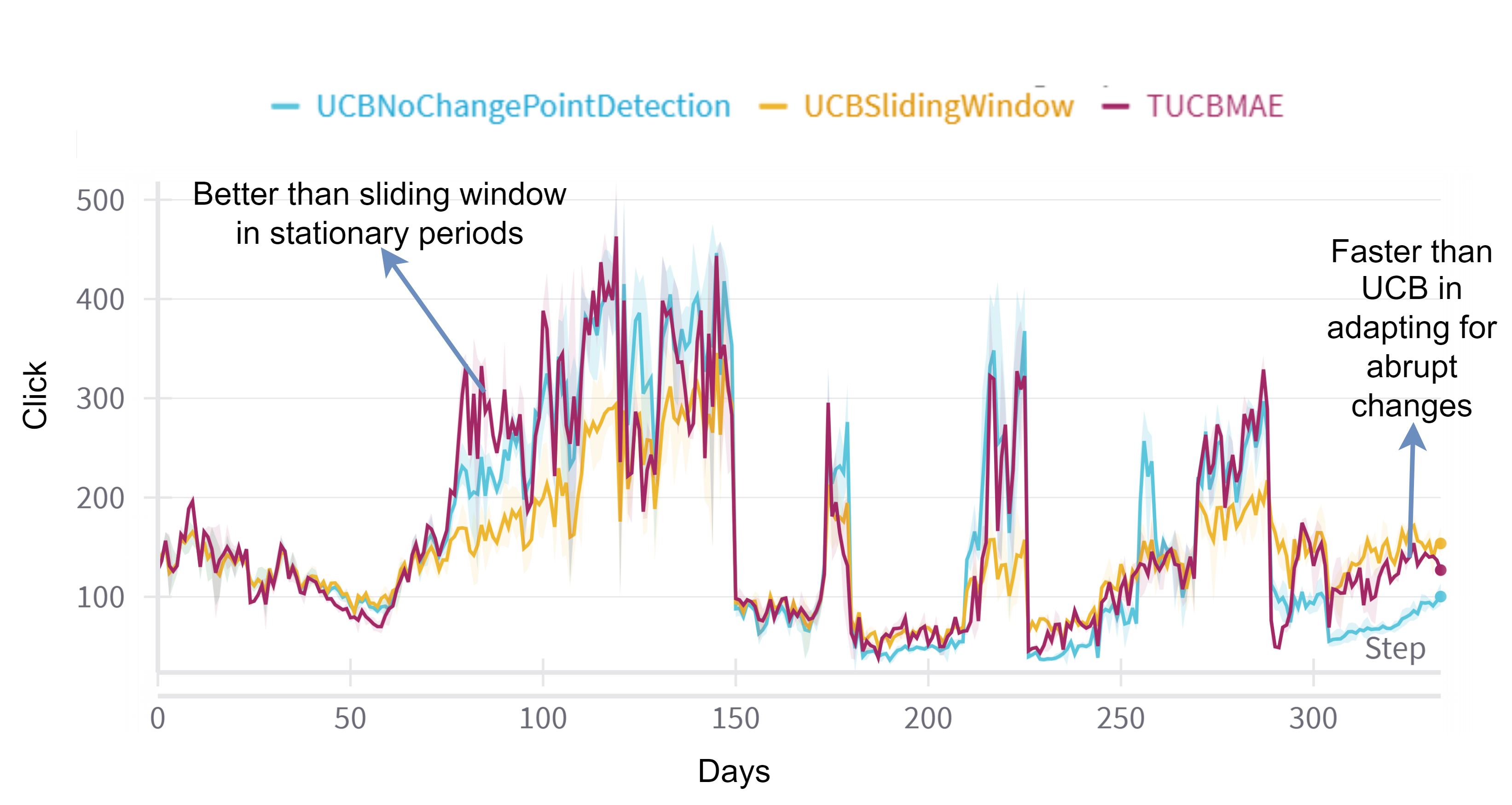}
\caption{Reward comparison for around 300 days for attendance management campaign.}
\label{rewardAakashi}
\end{figure}

\subsection{REWARD TYPES}

We consider two kinds of reward signals for budget allocation. The first choice is maximizing clicks which has been popularly used in pay per click advertisements \cite{gigli2024multi, nuara2022online}.  However, in businesses advertisers often aim at maximizing the number of conversions for campaigns which drives profitability. We observe the number of conversion per day is a very sparse signal often having a low value for many days which renders this signal inefficient to be estimated as a reward function and optimized directly. In order to optimize conversions we formulate $pseudo_{conversion}$ defined as follows:

\[ pseudo_{conversion} = \sum_{t=t-7}^{t} \text{click}_t \left( \frac{\sum_{t'=t-7}^{t} \text{conversion}_{t'}}{\sum_{t'=t-7}^{t} \text{click}_{t'}} \right)
\]

$pseudo_{conversion}$ calculates a weighted conversion rate based on the number of clicks for each day, scaled by the conversion rate over the past 7 days as depicted in Fig \ref{pseudoconversion}, capturing how effective ad campaigns are at driving conversion.

\begin{figure}[!h]
\centering
\includegraphics[width=0.48\textwidth]{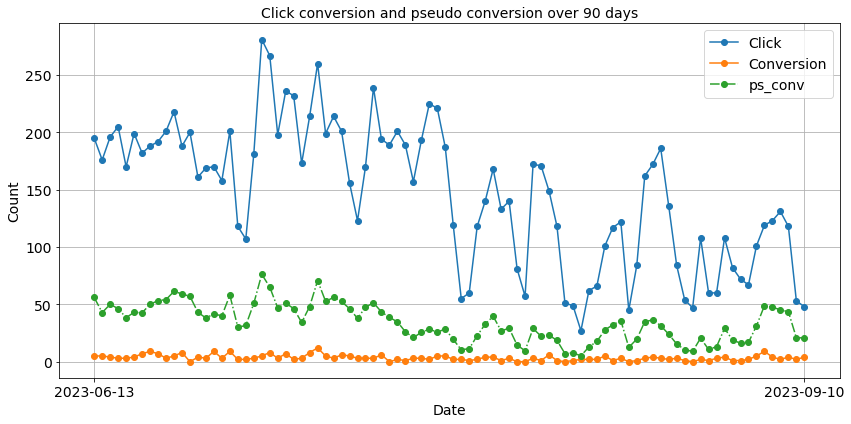}
\caption{Click, conversions and pseudo conversion of one sub-campaign from AI Prediction Tool Campaign}
\label{pseudoconversion}
\end{figure}

We perform comparison with the logged budget allocation of the human operator from the dataset for both clicks and pseudo conversions. The results are reported in Fig \ref{clickps}. TUCBMAE shows a performance improvement of at least 19\% compared to the human operator in terms of click and 5.8\% for pseudo conversions.

\subsection{Ablation Studies}

We perform ablation studies by studying the effect of different components of the proposed combinatorial bandit approach on AI Prediction Tool Campaigns with respect to clicks. TUCBMAENoSM represents a policy using  Targeted UCB with CPC as efficiency but no saturating mean. TUCBMAENoCPC is a policy without efficiency incentive for exploration. NoTUCBMAEWithCPC is a policy without targeted UCB for higher budget range but with CPC incentive for exploration along with normal UCB and saturating mean. The results are reported in Fig \ref{ablationstudy}. It can be clearly interpreted from the ablation study the targeted UCB has the highest contribution to performance gain as NoTUCBMAEWithCPC has the lowest reward. Additionally we observe the efficiency incentive provides performance boost. Finally, the ablation study shows all three components contribute to the performance improvement of the algorithm.

\begin{figure}[!h]
\centering
\includegraphics[width=0.48\textwidth]{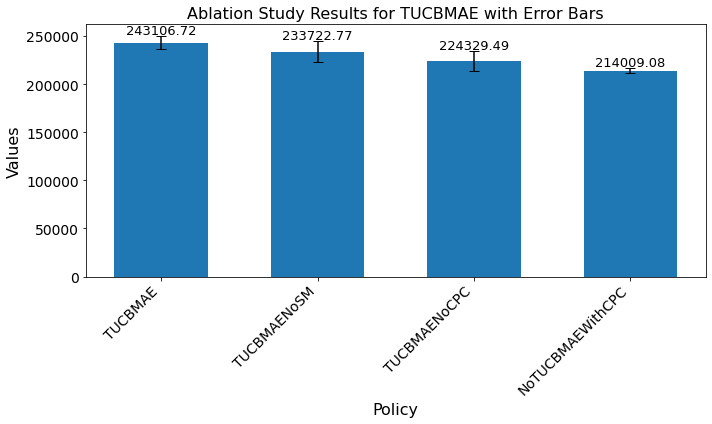}
\caption{Ablation study for TUCBMAE algorithm}
\label{ablationstudy}
\end{figure}

\subsection{Experiments on Criterio Dataset}

In order to demonstrate the compatibility of simulation environment with open source data popularly utilized in budget allocation algorithms we use criterio attribution dataset \cite{DiemertMeynet2017} with our simulation environment. This dataset does not provide a campaign structure we combine four random campaigns with ids [22589171, 884761, 18975823 and 29427842] to form a campaign group as followed in \cite{gigli2024multi}. Since the time horizon of this data is only 30 days and not expected to be non stationary we do not perform MAE change point detection. We compare the targeted UCB with saturating mean algorith with UCBGP (UCB with no change point detection) and TSGP (Thomson sampling with no change point detection). The results are reported in Table \ref{tab:criteriodataset}. The results demonstrate the the proposed strategy can lead to performance gain over UCB and TS exploration in stationary settings for standard dataset. 

\begin{table}[]
\caption{Results for opensource Criterio Attribution dataset.}
\label{tab:criteriodataset}
\resizebox{0.48\textwidth}{!}{%
\begin{tabular}{@{}cccc@{}}
\toprule
Metric & TUCB                    & TSGP                   & UCBGP                      \\ \midrule
Click  & \textbf{179447.23 $\pm$ 3991.89} & 160257.39 $\pm$ 11303.5 & 135784.005 $\pm$ 11836.84 \\ \midrule
Regret & \textbf{27499.34 $\pm$ 413.94}   & 46561.87 $\pm$ 11303.15 & 71035.26 $\pm$ 11836.84   \\ \midrule
CPC    & \textbf{353.06 $\pm$ 15.35}      & 397.71 $\pm$ 2.62       & 463.73 $\pm$ 14.02        \\ \bottomrule
\end{tabular}%
}
\end{table}

\section{Conclusion and Future Work}

The paper studies practical implication of deploying a combinatorial bandit algorithm for ad campaign budget management across multiple channels. We first construct a simulation environment capable of simulating real-logged data for long time horizon. We propose saturating mean and targeted UCB along with change point detection in combinatorial bandit for faster adaptation in non stationary environments. Our preliminary findings investigate the effects of non-stationarity in long-running digital advertising campaigns and the potential for improved adaptability. In future, we plan to formalize various types of non-stationary changes, including recurrent seasonal patterns, and further refine both the simulation environment and adaptation strategies to handle these challenges more effectively.

\begin{acks}
The authors thank Sony Biz Networks Corporation for providing data for the ad campaign from services NURO Biz and AKASHI.
\end{acks}


\balance
\bibliographystyle{ACM-Reference-Format} 
\bibliography{sample}


\clearpage

\section*{Appendix}

\subsection{Summary of Notations}

\begin{table}[h]
\centering
\begin{tabular}{|l|p{0.35\textwidth}|}
\hline
\textbf{Notation} & \textbf{Description} \\
\hline
$\beta_{j, t}$ & Exploration parameter for campaign $j$ at time $t$. \\
\hline
$\mathcal{R}_T(U)$ & Regret after $T$ rounds for algorithm $U$. \\
\hline
$\tilde{\mathcal{O}}(\cdot)$ & Big-O notation disregarding logarithmic factors. \\
\hline
$\hat{\mu}_{j,t-1}(x)$ & Estimated mean of campaign $j$ at time $t-1$ for budget $x$. \\
\hline
$\hat{\sigma}_{j,t-1}(x)$ & Estimated standard deviation of campaign $j$ at time $t-1$ for budget $x$. \\
\hline
$n_j(x)$ & True reward function. \\
\hline
$\hat{n}_j(x)$ & Realization of the GP representing reward of campaign $j$ at budget $x$. \\
\hline
$\gamma_T(\hat{n}_j)$ & Information gain from exploring campaign $j$ over $T$ rounds. \\
\hline
$\lambda$ & Variance of measurement noise of the reward functions $\hat{n}_j(x)$. \\
\hline
$\delta$ & Confidence parameter controlling probability bounds. \\
\hline
$M$ & Number of possible combinations of budgets explored. \\
\hline
$\mathcal{D}$ & Set of possible budgets or actions. \\
\hline
$T$ & Number of rounds or time steps. \\
\hline
$N$ & Number of campaigns or arms. \\
\hline
$S_t$ & Super-arm configuration at round $t$. \\
\hline
$a_j$ & Selected action or budget for campaign $j$. \\
\hline
$\mathcal{F}_\phi$ & A stationary phase. \\
\hline
$p_\phi$ & Break-point of phase $\mathcal{F}_\phi$. \\
\hline
$\theta_j$ & Arm efficiency. \\
\hline
\end{tabular}
\caption{Table of Notations}
\label{tab:notations}
\end{table}

\subsection{Detailed Proof}

\begin{lemma}[From \cite{10.5555/3104322.3104451}]
Given the realization of a GP $f(\cdot)$, the estimates of the mean $\hat{\mu}_{t-1}(x)$ and variance $\hat{\sigma}^2_{t-1}(x)$ for the input $x$ belonging to the input space $X$, for each $\beta \in \mathbb{R}^+$ the following condition holds:
\[
\mathbb{P} \left( \left| f(x) - \hat{\mu}_{t-1}(x) \right| \geq \sqrt{\beta} \, \hat{\sigma}_{t-1}(x) \right) \leq e^{-\frac{\beta}{2}},
\]
for each $x \in X$.
\end{lemma}

\begin{proof}
Let $r \sim \mathcal{N}(0,1)$ and $c \in \mathbb{R}^+$, we have:
\[
\mathbb{P}(r > c) = \frac{1}{\sqrt{2\pi}} e^{-\frac{c^2}{2}} \int_c^{\infty} e^{-\frac{(r-c)^2}{2} - c(r-c)} \, dr 
\]
\[
\leq e^{-\frac{c^2}{2}} \mathbb{P}(r > 0) = \frac{1}{2} e^{-\frac{c^2}{2}},
\]
since $e^{-c(r-c)} \leq 1$ for $r \geq c$. By the symmetry of the Gaussian distribution, we have:
\[
\mathbb{P}(|r| > c) \leq e^{-\frac{c^2}{2}}.
\]
Applying the above result to $r = \frac{f(x) - \hat{\mu}_{t-1}(x)}{\hat{\sigma}_{t-1}(x)}$ and $c = \sqrt{\beta}$ concludes the proof.
\end{proof}

\begin{proposition}
Let us consider an ABA problem over T rounds where the function $\hat{n}_j(x)$ is the realization of a GP, using TUCB-MAE algorithm with the following upper bound on the reward function $\hat{n}_j(x)$:
\[
u^{(n)}_{j,t-1}(x) := \hat{\mu}_{j,t-1}(x) + \sqrt{\beta_{j, t}} \hat{\sigma}_{j,t-1}(x)
\]
with probability at least $1 - \delta$, it holds:
\[
\mathcal{R}_T(U) = \tilde{\mathcal{O}} \left( \sqrt{TN \sum_{j=1}^{N} \gamma_T(\hat{n}_j)} \right),
\]
where the notation $\tilde{\mathcal{O}}\left( \cdot \right)$ disregards the logarithmic factors. 

\textit{Proof : } In ABA-UCB, we assume the number of clicks $\hat{n}_j(x)$ of a campaign $A_j$ is the realization of a GP over the budget space $x$. Using the selected input $a_j$ and the corresponding observations $\tilde{n}_{j,h} = \tilde{n}_j(a_j, h)$ for each $h \in \{1, \ldots, t-1\}$, the GP provides the estimates of the mean $\hat{\mu}_{j,t-1}(x)$ and variance $\hat{\sigma}_{j,t-1}^2(x)$ for each $x$. The sampling phase is based on the upper bounds on the number of rewards formally:

\[
u^{(n)}_{j,t-1}(x) := \hat{\mu}_{j,t-1}(x) + \sqrt{\beta_{j, t}} \hat{\sigma}_{j,t-1}(x), \tag{A.1}
\]

where $x$ is the cost,n denotes the round and j is the campaign.

Applying Lemma 1 to Equation (A.1) for a generic arm $a$ and $b = b_t$ we have:
\[
\mathbb{P}\left[ \left| \hat{n}_j(x) - \mu_{j,t-1}(x) \right| > \sqrt{\beta_{j, t}} \sigma_{j,t-1}(x) \right] \leq e^{-\frac{\beta_{j, t}}{2}}.
\]
In the execution of the \texttt{ABA-UCB} algorithm, after $t-1$ rounds, each arm can be chosen a number of times from $0$ to $t-1$. Applying the union bound over the rounds $\left(t \in \{1, \dots, T\}\right)$, the campaigns $\left(j \in \{1, \dots, N\}\right)$ and the available action in each campaign $\left( a \in \mathcal{D} \right)$, and exploiting Lemma (1), we obtain:
\[
\mathbb{P} \left[
\bigcup_{t \in \{1, \dots, T\}} \bigcup_{j \in \{1, \dots, N\}} \bigcup_{a \in \mathcal{D}} 
\left( \left| \hat{n}_j(x) - \mu_{j,t-1}(x) \right| > \sqrt{\beta_{j, t}} \sigma_{j,t-1}(x) \right)
\right]
\]
\[
\leq \sum_{t=1}^{T} \sum_{j=1}^{N} M e^{-\frac{\beta_{j, t}}{2}}.
\]

Where M represents the number of possible combinations of budget that the algorithm can explore. The larger the number of budget, the more difficult it becomes to explore the space effectively, hence the need for more exploration. For each time t, for each campaign j, and for each action a $\in$ D, the probability of the event occurring is bounded by the size of the action set M times the exponential decay.

Thus, choosing $\beta_{j, t} = 2k_j \log \left( \frac{\pi^2 NMt^2}{3 \delta} \right)$, where $k_j = (1 - \theta_j)$ we obtain:

\[
\sum_{t=1}^{T} \sum_{j=1}^{N} M e^{-\frac{\beta_{j, t}}{2}} 
= \sum_{t=1}^{T} \sum_{j=1}^{N} M (\frac{3 \delta}{\pi^2 NM t^2})^{k_j} 
\leq \frac{\delta}{2N} \sum_{j=1}^{N} \left( \frac{6}{\pi^2} \sum_{t=1}^{\infty} \frac{1}{t^2} \right) \leq \frac{\delta}{2}.
\]

$e^{alogb} = b^a$ and $k_j \in (0,1]$.
\\

So, the total probability of deviating significantly from the true values (across campaigns, time steps, and actions) is less than or equal to $\frac{\delta}{2}$. This ensures high probability guarantee for the entire bound over all time steps, campaigns, and actions, ensuring the algorithm's decisions are made with a high level of confidence. Therefore, the event that at least one of the upper bounds over the actual reward does not hold has probability less than $\delta$.

Assume to be in the event that all the previous bounds hold. The instantaneous pseudo-regret \( reg_t \) at round \( t \) satisfies the following inequality:
\[
reg_t = r_{\mu}^* - r_{\mu}(S_t) \leq r_{\mu}^* - r_{\bar{\mu}_t}(S_t) + r_{\bar{\mu}_t}(S_t) - r_{\mu}(S_t),
\]
where
\[
\bar{\mu}_t := \left(u_{1,t-1}^{(n)}(a_1), \dots, u_{N,t-1}^{(n)}(a_M) \right)
\]
is the vector composed of all the upper bounds of the different actions (of dimension \( NM \)).

Let us recall that, given a generic superarm \( S \), if all the elements of a vector \( \mu \) are larger than the ones of \( \mu' \), the following holds:
\[
r_{\mu}(S) \geq r_{\mu'}(S).
\]

Let us focus on the term \( r_{\bar{\mu}_t}(S_t) \). The following inequality holds:
\[
r_{\bar{\mu}_t}(S_t) \geq r_{\bar{\mu}_t}(S_{\mu}^*) \geq r_{\mu}(S_{\mu}^*) \geq r_{\mu}(S_{\mu}^*) = r_{\mu}^*,
\tag{A.3}
\]
where \( S_{\mu}^* = \arg \max_{S \in \mathcal{S}}(r_{\mu}(S)) \) is the super-arm providing the optimum expected reward when the expected rewards are \( \mu \). Thus, we have:
\[
reg_t \leq r_{\bar{\mu}_t}(S_t) - r_{\mu}(S_t).
\]

\[
\leq r_{\bar{\mu}_t}(S_t) - r_{\mu_t}(S_t) + r_{\mu_t}(S_t) - r_{\mu}(S_t),
\]
where 
\[
\mu_t := (\hat{\mu}_{1,t-1}(a_1), \dots, \hat{\mu}_{N,t-1}(a_M))
\]
is the vector composed of the estimated average rewards for each arm \( a \in \mathcal{D} \).

\[
r_{\bar{\mu}_t}(S_t) - r_{\mu_t}(S_t) = \sum_{j=1}^{N} \left( u_{j,t-1}^{(n)}(a_j,t) - \hat{\mu}_{j,t-1}(a_j,t) \right)
\]
\[
= \sum_{j=1}^{N} \left( \hat{\mu}_{j,t-1}(a_j,t) + \sqrt{\beta_{j, t}} \hat{\sigma}_{j,t-1}(a_j,t) - \hat{\mu}_{j,t-1}(a_j,t) \right)
\]
\[
= \sum_{j=1}^{N} \sqrt{\beta_{j, t}} \hat{\sigma}_{j,t-1}(a_j,t)
\]
\[
\leq \sum_{j=1}^{N} \sqrt{\beta_{j, t}} \max_{a \in \mathcal{D}} \hat{\sigma}_{j,t-1}(a)
\]

Let us focus on the term \( r_{\mu_t}(S_t) - r_{\mu}(S_t) \):

\[
r_{\mu_t}(S_t) - r_{\mu}(S_t) = \sum_{j=1}^{N} \left( \hat{\mu}_{j,t-1}(a_j,t) - \hat{n}_j(a_j,t) \right)
\]

\[
\leq \sum_{j=1}^{N} \sqrt{\beta_{j, t}} \max_{a \in \mathcal{D}} \hat{\sigma}_{j,t-1}(a)
\]

Given the UCB Gurantee.

Summing up the two terms we have:
\[
reg_t \leq 2 * \sum_{j=1}^{N} \sqrt{\beta_{j, t}} \max_{a \in \mathcal{D}} \hat{\sigma}_{j,t-1}(a)
\]

We now need to upper bound \( \hat{\sigma}_{j,t-1}(a) \).Using Lemma 5.3 in \cite{10.5555/3104322.3104451}, under the Gaussian assumption we can express the information gain provided by the observations \( n_{t-1} = (\tilde{n}_{j,1}, \dots, \tilde{n}_{j,t-1}) \) corresponding to the sequence of actions \( (a_{j,1}, \dots, a_{j,t-1}) \) as:

\[
IG(n_{t-1} | \hat{n}_j) = \frac{1}{2} \sum_{h=1}^{t-1} \log \left( 1 + \frac{\hat{\sigma}^2_{j,h}(a_j,h)}{\lambda} \right).
\]

Since \( b_h \) is non-decreasing in \( h \), we can write:

\[
\sigma^2_{j,h}(a_j,h) = \lambda \left[ \frac{\hat{\sigma}^2_{j,h}(a_j,h)}{\lambda} \right] \leq \frac{\log \left( 1 + \frac{\hat{\sigma}^2_{j,h}(a_j,h)}{\lambda} \right)}{\log \left( 1 + \frac{1}{\lambda} \right)},
\tag{A.4}
\]

since \( s^2 \leq \frac{\log \left( 1 + s^2 \right)}{\log \left( 1 + \frac{1}{\lambda} \right)} \) for all \( s \in \left[0, 1 \right], \) and \( \frac{\hat{\sigma}^2_{j,h}(a_j,h)}{\lambda} = \frac{k(a_j,h,a_j,h)}{\lambda} \leq \frac{1}{\lambda}. \)

Since Equation (A.4) holds for any \( a \in \mathcal{D} \), then it also holds for the action \( a_{\max} \) maximizing the variance \( \sigma^2_{j,h}(a_{j,h}) \) in \( \hat{n}_j \) defined over \( \mathcal{D} \). Thus, using the Cauchy-Schwarz inequality, we obtain:

\[
\mathcal{R}_T^2(U) \leq T \sum_{t=1}^{T} \text{reg}^2_t
\]

\[
\leq T \left( 2 \sum_{j=1}^{N} \sqrt{\beta_{j, t}} \max_{a \in \mathcal{D}} \hat{\sigma}_{j,t-1}(a) \right)^2
\]

\[
\leq 4 T \left\{ \sum_{j=1}^{N} \sum_{t=1}^{T} \beta_{j, t} \left[ \max_{a \in \mathcal{D}} \frac{\log \left( 1 + \frac{\hat{\sigma}^2_{j,n-1}(a)}{\lambda} \right)}{\log \left( 1 + \frac{1}{\lambda} \right)} \right] \right\}
\]

\[
\leq 4T \sum_{j=1}^{N} \sum_{t=1}^{T} \left[ \beta_{j, t} \max_{a \in \mathcal{D}} \frac{\log \left( 1 + \frac{\hat{\sigma}^2_{j,n-1}(a)}{\lambda} \right)}{\log \left( 1 + \frac{1}{\lambda} \right)} \right]
\]

As $k_j$ is between 0 and 1. 

\[
\leq 4T \sum_{j=1}^{N} \sum_{t=1}^{T} \left[ \beta_{t} \max_{a \in \mathcal{D}} \frac{\log \left( 1 + \frac{\hat{\sigma}^2_{j,n-1}(a)}{\lambda} \right)}{\log \left( 1 + \frac{1}{\lambda} \right)} \right]
\]

\[
\leq 4T\beta_{T} \sum_{j=1}^{N} \sum_{t=1}^{T} \left[  \max_{a \in \mathcal{D}} \frac{\log \left( 1 + \frac{\hat{\sigma}^2_{j,n-1}(a)}{\lambda} \right)}{\log \left( 1 + \frac{1}{\lambda} \right)} \right]
\]

\[
\leq 4T\beta_{T} \left\{ \frac{1}{\log \left( 1 + \frac{1}{\lambda} \right)} \sum_{j=1}^{N} \underbrace{
\sum_{t=1}^{T} \max_{a \in \mathcal{D}} \log \left( 1 + \frac{\hat{\sigma}^2_{j,n-1}(a)}{\lambda} \right)
}_{\gamma_T(\hat{n}_j)} \right\}
\]

\[
\leq 4T\beta_{T} \left\{ \frac{1}{\log \left( 1 + \frac{1}{\lambda} \right)} \sum_{j=1}^{N} \gamma_T(\hat{n}_j) \right\}
\]

where, \( \lambda \) is variances of the measurement noise of the reward functions \( \hat{n}_j(\cdot) \).

Equivalently, with probability at least \( 1 - \delta \), it holds:

\[
\mathcal{R}_T(U) = \tilde{\mathcal{O}} \left( \sqrt{TN \sum_{j=1}^{N} \gamma_T(\hat{n}_j)} \right),
\]

If we explore values of \( x_j \leq x_{j\max,t} \), by monotonicity, we have:
\[
\hat{n}_j(x_{j}^*) \geq \hat{n}_j(x_{j\max,t}) \geq \hat{n}_j(x_{j,t}).
\]
This means that exploring in this region incurs unnecessary regret because we are not gaining new information about potentially better actions.

\paragraph{Information Gain \( \gamma_T(\hat{n}_j) \):} The information gain measures how much we learn from exploring the actions. Exploring in regions where \( x_j \leq x_{j\max} \) leads to little or no information gain due to monotonicity, because it only confirms what is already known — that lower values of \( x \) will not perform better. Therefore, this exploration adds to regret without yielding useful information.

\paragraph{Reduced Exploration Space:} By restricting exploration to values \( x_j > x_{j\max,t} \), the effective space of arms to explore is reduced. This reduces the total information gain \( \gamma_T(\hat{n}_j) \), which in turn reduces the regret bound. Specifically, if we denote the restricted exploration space by \( X_j^+ \), we have:
\[
\gamma_T(\hat{n}_j, X_j^+) \leq \gamma_T(\hat{n}_j).
\]

Thus, the regret bound becomes:
\[
R_T(U^+) = O \left( \sqrt{TN \sum_{j=1}^{N} \gamma_T(\hat{n}_j, X_j^+)} \right),
\]
 Since \( \gamma_T(\hat{n}_j, X_j^+) \leq \gamma_T(\hat{n}_j) \), this shows that the regret is already reduced by restricting exploration to values \( x > x_{\max,t} \).

\end{proposition}

\subsection{Hyper-parameters}

The hyper-parameters used in our experiments are described in Table \ref{genhp} and \ref{campaignhyperparameters}. We only tune two hyper-parameters per campaign group which is the exploration parameter $\beta$ and the change point detection threshold $\tau$. The values are reported in Table \ref{campaignhyperparameters}. For criterio dataset since there is no change point detection we only tune $\beta = 2$. The other hyperparameters are reported in Table \ref{genhp}.

\begin{table}
\centering
\caption{General Hyperparameter}
\label{genhp}
\begin{tblr}{
  width = \linewidth,
  colspec = {Q[558]Q[294]},
  hlines,
}
Parameter & Value\\
$\epsilon$ & $\mathcal{N}$(0,0.1)\\
$window_{length}$ & 7\\
$T_p$ & 20\\
$l$ & 1.0\\
$B$ & 500
\end{tblr}
\end{table}

\begin{table}
\centering
\caption{Campaign specific hyperparameters}
\label{campaignhyperparameters}
\begin{tblr}{
  width = \linewidth,
  colspec = {Q[625]Q[156]Q[133]},
  hlines,
}
Product Type & $\beta$ & $\tau$\\
Attendance Management System & 100 & 10\\
Prediction Analysis Tool & 2 & 4\\
Internet Service Provider & 100 & 4\\
Product17276 & 2 & 10\\
Product1598 & 50 & 10
\end{tblr}
\end{table}


\end{document}